\documentclass[letterpaper, 10 pt, conference]{ieeeconf}  
\IEEEoverridecommandlockouts                              
\usepackage{epsfig} 
\usepackage{times} 
\usepackage{hyperref}
\usepackage[utf8]{inputenc}
\usepackage{mathtools}
\usepackage{xcolor}

\usepackage{mathrsfs}
\usepackage{comment} 
\usepackage{algorithm,algorithmic}
\usepackage{empheq}
\usepackage{dsfont}

\usepackage{amssymb}
\usepackage{amsmath}

\newtheorem{theorem}{\textbf{Theorem}}

\usepackage{mathtools} 

\newcommand{\tf}[1]{\mathbf{#1}}

\newcommand{\SFpair}{\left\{\Phix,\Phiu\right\}}
\newcommand{\Phix}{\tf \Phi_x}
\newcommand{\Phiu}{\tf \Phi_u}

\newcommand{\Phixt}{\Phi_x}
\newcommand{\Phiut}{\Phi_u}

\DeclareMathOperator*{\argmin}{arg\,min}
\newcommand{\sectionnew}[1]{\vspace{0.5mm}

\noindent{\textbf{#1}}}
\newcommand{\sectionnewit}[1]{\vspace{0.5mm}\noindent{\underline{\textit{#1}}}}

\title{\Large \bf
Robust Disturbance Rejection for Robotic Bipedal Walking: \\ System-Level-Synthesis with Step-to-step Dynamics Approximation}
    \author{Xiaobin Xiong, Yuxiao Chen, and Aaron D. Ames
    \thanks{*This work is supported by NSF grant 1924526 and 1923239.}
\thanks{The authors are with the Department of Mechanical and Civil Engineering, California Institute of Technology, Pasadena, CA, {\tt\small \{xxiong, yxchen, ames\}@caltech.edu}}}
\begin{document}
\maketitle
\thispagestyle{empty}
\pagestyle{empty}

\begin{abstract}
We present a stepping stabilization control that addresses external push disturbances on bipedal walking robots. The stepping control is synthesized based on the step-to-step (S2S) dynamics of the robot that is controlled to have an approximately constant center of mass (COM) height. We first learn a linear S2S dynamics with bounded model discrepancy from the undisturbed walking behaviors of the robot, where the walking step size is taken as the control input to the S2S dynamics. External pushes are then considered as disturbances to the learned S2S (L-S2S) dynamics. We then apply the system-level-synthesis (SLS) approach on the disturbed L-S2S dynamics to robustly stabilize the robot to the desired walking while satisfying the kinematic constraints of the robot. We successfully realize the proposed approach on the walking of the bipedal robot AMBER and Cassie subject to push disturbances, showing that the approach is general, effective, and computationally-efficient for robust disturbance rejection.
\end{abstract}

\section{Introduction}
Bipedal robots are showing premises of entering real life to perform meaningful tasks in human society \cite{johnson2015team}. Various methodologies such as the zero-moment-point (ZMP) \cite{kajita2003biped} and the hybrid-zero-dynamics (HZD) \cite{westervelt2003hybrid, grizzle2014models} framework have been proposed in the literature to generate bipedal robotic walking. These methods typically decompose the walking controllers into two components: walking trajectory planning in the configuration or state-space and low-level output stabilization on the trajectories. For instance, the HZD framework relies on offline parameter optimization to generate periodic orbits in a low-dimensional manifold that is hybrid invariant; feedback controllers are then synthesized to make the manifold attractive and thus render stable orbits. 

Despite the similarities and differences in various approaches to walking generation, the commonality is on the need to provide robustness for the controlled walking to model discrepancy and external disturbances \cite{stephens2007humanoid,learningPushRecovery, rebula2007learning}. The overarching goal is to prevent the robot from falling as this can lead to catastrophic damages to the hardware. Typically, robustness evaluation is done via pushes on the walking robot. The synthesis of push-robust controllers on walking robots, however, remains an open and challenging problem; this is because bipedal robots are high-dimensional and their walking dynamics are hybrid in nature.

Canonical studies on this problem mostly come from the robotics community, where the push-robust controllers \cite{stephens2007humanoid} are heuristically decomposed into three general principles: the hip strategy, ankle strategy, and the stepping stabilization \cite{pratt2006capture, sugihara2009standing}. The synthesis of heuristic controllers when experiencing push disturbances are oftentimes based on simple dynamic models that approximate the robot dynamics. As a result, the final control formulations thus have little characterization of robustness and typically do not guarantee the satisfaction of kinematic feasibility; e.g, the desired step size may exceed the kinematic limit of the robot. 


 \begin{figure}[t]
      \centering
      \includegraphics[width = .85\columnwidth]{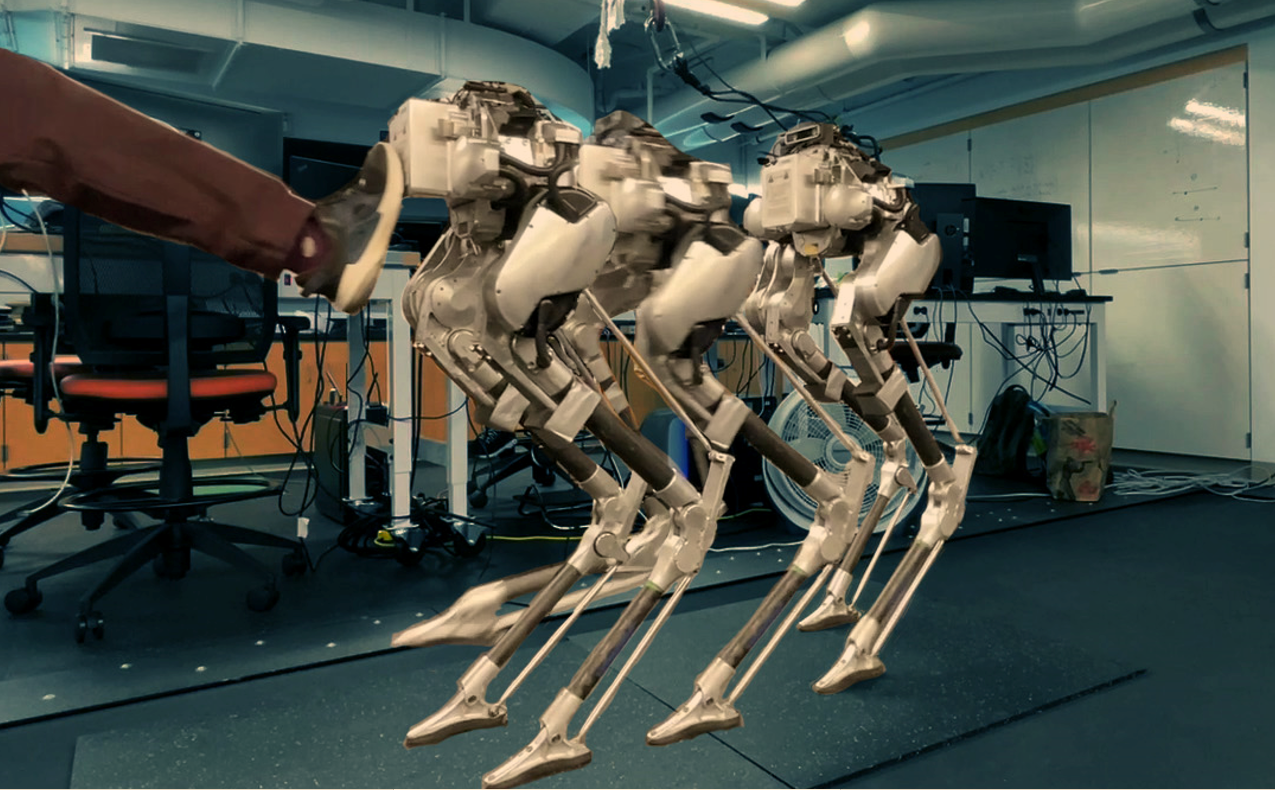}
      \caption{Push rejection on the robot Cassie using the proposed approach.}
      \label{fig:overview}
\end{figure}


In this paper, we focus on the in-depth study of robust bipedal stabilization that can reject push disturbances while respecting the kinematic feasibility of the physical robot. We first learn a step-to-step (S2S) dynamic model from the push-free robotic walking, which is generated using existing framework in \cite{xiong20213d}. The learned S2S (L-S2S) dynamics is a discrete dynamical system with the walking step size being the inputs. Periodic walking behaviors can then be characterized on the L-S2S dynamics. Next, we approximate the external push as an external disturbance to the L-S2S dynamics. Finally we apply the system-level-synthesis (SLS) approach on the L-S2S dynamics to generate desired step sizes that stabilize the walking subject to the push disturbance. Importantly, since the state and input constraints (i.e., system constraints) on the L-S2S dynamics can be encoded in the SLS \cite{chen2019system}, kinematic constraints on the robot state and step sizes are respected. 

The proposed approach can be viewed as an extension of stepping controllers \cite{xiong20213d, xiong2020ral, pratt2012capturability, bhounsule2020approximation} that uses the S2S dynamics approximation to plan foot steps for bipedal walking stabilization. Here, the approximation of the S2S is directly learned from the walking data of the robot. More importantly, we find that the same parameterized L-S2S dynamics sufficiently approximates the actual S2S dynamics of two different robots with trivial model discrepancies. The L-S2S dynamics is also linear and thus creates a framework for linear controllers to be formally applied to hybrid nonlinear dynamical systems such as walking robots. 

The application of the SLS also shows a great advantage over previous synthesis such as deadbeat control \cite{xiong20213d,pratt2012capturability, bhounsule2016dead} and Linear Quadratic Regulators (LQR) \cite{boyd1991linear} in \cite{xiong2020ral} in bipedal stepping. The feedback controller obtained from the SLS includes the system constraints while optimizing a cost function of the states and inputs, which is very important in practice since the system constraints on the S2S dynamics are the kinematic constraints of the robot. It also renders the closed-loop system finite-impulse response (FIR), which rejects disturbances within finite footsteps. Additionally, the SLS optimization is convex, thus easily implementable, and results in a closed-form stepping controller.

We realize the proposed approach on two bipedal robots AMBER and Cassie, with the result being walking under bounded push disturbances. The SLS controller not only stabilizes the disturbed robotic walking in finite steps, but also respects the kinematic constraints in all steps. The realized walking is also more accurate in terms of velocity tracking. The rest of the paper is organized as follows. Section \ref{sec:prelim} introduces the recent development of stepping controller based on the S2S dynamics of walking. Section \ref{sec:RS2S} presents the system identification of the robotic S2S dynamics, and Section \ref{sec:push} formulates the robust stepping stabilization problem for push disturbances. Then, we apply the SLS approach on the stepping stabilization in Section \ref{sec:SLS}. Finally, we evaluate the approach in Section \ref{sec:results}, and conclude the paper in Section \ref{sec:conclude}.







\section{Preliminary: Stepping Control Based on S2S Dynamics}
\label{sec:prelim}
We now introduce the stepping controller that is based on the approximation of the step-to-step (S2S) dynamics. We first define the S2S dynamics of bipedal walking and then introduce the Hybrid Linear Inverted Pendulum (H-LIP) based walking synthesis \cite{xiong20213d}: the S2S dynamics of the H-LIP is used to approximate the robot S2S dynamics. A state-feedback stepping controller, namely H-LIP stepping \cite{xiong2020ral, xiong20213d}, is then synthesized to discretely control the S2S state of the robot to achieve desired walking behaviors.

\subsection{Hybrid Dynamics and Step-to-step Dynamics}
Bipedal robotic walking is typically modeled as a hybrid dynamical system \cite{grizzle2014models} that undergoes continuous dynamics and discrete transitions. The robot is modeled as a rigid body system, and the continuous dynamics can be derived from the Lagrangian mechanics. The impact between the foot and the ground is typically modeled as plastic impact. The hybrid dynamics can be briefly described by 
\begin{align}
  &  \dot{x} = f_n(x, \tau),\\
   & x^+ = \Delta_{n \rightarrow n+1}(x^-),
\end{align}
where $x$ is the state of the robot, $f_n$ represents the nonlinear dynamics in the domain denoted by $n$, $\tau$ stands for the actuation, $\Delta_{n \rightarrow n+1}$ represents the discrete transition from the domain $n$ to the domain $n+1$, and the superscript $^{+/-}$ indicate the state after/before the discrete transition. 

The hybrid dynamics can be converted to a discrete step-to-step (S2S) dynamics \cite{xiong2020ral, xiong20213d} which then facilitates step planning to stabilize walking. Let $x^-$ denote the pre-impact state before the swing foot strikes the ground. Assume the existence of the next foot-strike. The current pre-impact state with the continuous joint actuation will determine the next pre-impact state:
$
    x^-_{k+1} = \mathcal{P}(x^-_k,  \tau(t)),
$
which is referred to as the S2S dynamics of the full state of the robot with $k$ being the index of the step. $t$ represents for the continuous time. For walking, we mainly care about the evolution of the weakly-actuated horizontal center of mass (COM) state: the horizontal COM position w.r.t. the stance foot $p$ and its velocity $v$, the S2S dynamics of which is:
\begin{equation}
\label{eq:robotS2S}
    \mathbf{x}_{k+1} = \mathcal{P}_{\mathbf{x}}(x^-_k,  \tau(t)).
\end{equation}
$\mathbf{x} = [p, v]^T$ stands for the pre-impact horizontal COM state. In the latter, we will denote Eq. \eqref{eq:robotS2S} as the S2S dynamics of the robot since the horizontal COM state is most effectively stabilized via the S2S dynamics. 

 \begin{figure}[t]
      \centering
      \includegraphics[width = .85\columnwidth]{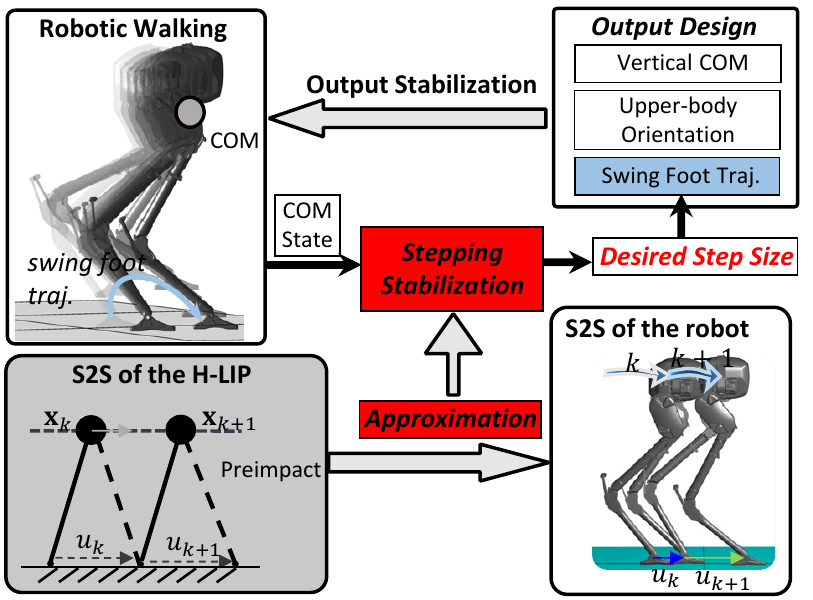}
      \caption{Illustration of the walking generation based on the S2S dynamics. This paper mainly focuses on the red boxed components: the S2S dynamics approximation and the synthesis of the robust stepping stabilization.}
      \label{fig:prelim}
\end{figure}

\subsection{S2S Dynamics Approximation via Hybrid-LIP}
The continuous dynamics of the robot is highly nonlinear, thus the analytical form of the S2S dynamics of the robot is difficult to obtain. \cite{xiong20213d} proposed an approximation to the robot S2S dynamics by the S2S dynamics of the Hybrid Linear Inverted Pendulum (H-LIP) while controlling the vertical COM of the robot to be (approximately) constant. The model discrepancy can be treated as bounded disturbance for a wide range of practically-realizable walking behaviors. The S2S dynamics of the H-LIP is:
\begin{equation}
\label{eq:HLIPs2s}
    \mathbf{x}^H_{k+1} = A \mathbf{x}^H_k + B u^H_k,
\end{equation}
where $\mathbf{x}^H = [p^H, v^H]^T$ and $u^H$ are the pre-impact state and the step size of the H-LIP, respectively. The derivation and expressions of $A$ and $B$ are detailed in \cite{xiong20213d}. Using the H-LIP as an \textit{approximation}, the robot S2S dynamics is rewritten as:
\begin{align}
\label{eq:RobotS2Sapprox}
\mathbf{x}_{k+1} = A \mathbf{x}_{{k}}  +  B u_{k} + w_m,
\end{align}
where $w_m := \mathcal{P}_{\mathbf{x}}(x^-_k,  \tau(t))- A \mathbf{x}_{{k}}  -  B u_{k}$, and $u$ is the step size of the robot. $w_m$ is the model discrepancy: the integration of the difference of the continuous horizontal COM dynamics difference between the robot and the H-LIP over a step. Given a range of walking behaviors, corresponding $w_m$ belongs to a bounded set $\mathbf{W}_m$. Applying the H-LIP based stepping (with $K$ being static feedback gain matrix)
 \begin{equation}
 \label{eq:HLIPstepping}
u = u^{H} + K (\mathbf{x}_{k} - \mathbf{x}_{k}^{H} )
\end{equation}
yields the \textit{error dynamics}: $
    \mathbf{e}_{k+1} = (A+BK) \mathbf{e}_{k}  + w_m,
$
where $\mathbf{e} := \mathbf{x} - \mathbf{x}^H$ is the error state. A selection of $K$ to make $A+BK$ stable then drives $\mathbf{e}$ to converge to a minimal robust positively invariant set \cite{1582504} $\mathbf{E}$, i.e., if $\mathbf{e}_k \in \mathbf{E}$, $\mathbf{e}_{k+1} \in \mathbf{E}$. 
In \cite{xiong20213d,xiong2020ral}, Eq. \eqref{eq:HLIPstepping} is applied to realize walking behaviors on the robot using deadbeat controllers or linear quadratic regulators (LQR) to design the feedback gain $K$. 



\subsection{Output Construction and Stabilization}
The H-LIP stepping provides the desired step sizes for the robot to realize. The desired walking trajectories (output trajectories) of the robot are then constructed to realize a constant COM height, a time-based vertical swing foot trajectory to periodically lift-off and strike the ground, and finally the horizontal swing foot trajectory to realize the desired step size. 3D robotic walking is decoupled into sagittal and lateral planes. In each plane, the horizontal COM states are stabilized individually to its desired walking. Additional desired output trajectories include these of the upper body configuration and swing foot orientation. The output stabilization can be realized via optimization-based controllers such as the control Lyapunov function based Quadratic Programs \cite{ames2014rapidly}. Fig. \ref{fig:prelim} illustrates the previous framework for walking realization on bipedal robots.

\textbf{Remark:}  In this paper, we will focus on the improvement of the stepping stabilizing controller and keep the output construction and output stabilizing controllers in \cite{xiong20213d} intact, including the parameters of the trajectories and feedback gains in the output stabilization.

\section{Learning S2S Dynamics Approximation}
\label{sec:RS2S}

The H-LIP based approximation was proposed as a general model for bipedal walking, thus the model difference $w_m$ between the H-LIP and an actual physical robot may not be well characterized. In this section, we consider a data-driven approximation of the actual S2S dynamics of the robot to reduce the model discrepancy. Given a range of walking behaviors under the H-LIP based approach \cite{xiong20213d}, we formulate a linear program to learn a linear dynamic model that best approximates the robot S2S dynamics with minimum bounds on the model discrepancy. Then, we characterize the periodic walking behaviors on the learned S2S (L-S2S) dynamics, which are later used for synthesizing robust stepping controllers on the robot.    


\subsection{Data-Driven S2S Dynamics}
Assume the walking dataset that covers a wide range of velocities is given by the H-LIP approach in \cite{xiong20213d} with the same feedback gains and gait parameters. Consider the \textit{actual S2S dynamics} of the robot is in the following form:
\begin{equation} 
\label{eq:learnedS2S}
    \mathbf{x}_{k+1} = \bar{A} \mathbf{x}_k + \bar{B} u_k + \bar{C} + \epsilon, 
\end{equation}
where $\bar{A} \in \mathbb{R}^{2\times 2}$, $\bar{B} \in \mathbb{R}^{2\times 1}$, and $\bar{C} \in \mathbb{R}^{2\times 1}$ are the parameters of the L-S2S to be learned from the data. $\epsilon\in \mathbb{R}^{2\times 1}$ is the residual. The above equation can be transformed into a linear equation:
$o_k q_\text{L-S2S} + \epsilon = \mathbf{x}_{k+1}$ with 
$q_\text{L-S2S} = [\bar{A}_{(1,1)}, \bar{A}_{(1,2)},\bar{A}_{(2,1)}, \bar{A}_{(2,2)},  \bar{B}_1, \bar{B}_2, \bar{C}_1, \bar{C}_2], \nonumber
$
where the subscripts indicate the corresponding element in the matrix, and
$
o_k = \begin{bmatrix} p_{k} & v_k& 0& 0& u_k& 0& 1& 0 \\
0&0&  p_{k}& v_k&0&  u_k& 0& 1 \end{bmatrix}.
$
For the purpose of obtaining a minimum bound on the residual, instead of solving a least-square problem, we solve a $\mathcal{L}_\infty$ regression via a linear program:
\begin{align}
\label{eq:L-S2S}
[\bar{A}, \bar{B}, \bar{C}, d^*] &= \argmin \limits_{q_\text{L-S2S}\in\mathbb{R}^8, d\in\mathbb{R}^2} \mathds{1}^\intercal d\\
\mathrm{s.t.}~ \forall k,& -d\le o_k q_\text{L-S2S} -\mathbf{x}_{k+1}\le d. \nonumber
\end{align}
Solving this optimization yields the linear dynamics in Eq. \eqref{eq:learnedS2S} with 
$
    \epsilon \in \mathbf{D} :=  \mathop{\otimes}\nolimits_i [-d^*_i, d^*_i], 
$
where $\mathbf{D}$ represents the polytopic set that contains all $\epsilon$, and $\otimes$ denote the Kronecker product. In practice, we will show that $d^*$ and thus $\mathbf{D}$ are very small. Now, $\epsilon$ is the bounded model discrepancy, which can be treated as both state and input \textit{independent}. The \textit{learned S2S dynamics} (L-S2S) of the robot then is: 
\begin{equation} 
\label{eq:rS2S}
    \mathbf{x}_{k+1} = \bar{A} \mathbf{x}_k + \bar{B} u_k + \bar{C}.
\end{equation}

\vspace{-1mm}
\sectionnewit{Comparison:} The L-S2S dynamics is different than the linear S2S dynamics of the H-LIP in Eq. \eqref{eq:HLIPs2s}. Unlike the H-LIP, we no longer have a physical model or hybrid dynamics that results in the L-S2S dynamics. The term $\bar{C}$ captures the dynamics effect on the stepping to some extend. For instance, a step with $\mathbf{x}_k = [0,0]^T$ and $u_k = 0$ results in $\mathbf{x}_{k+1} =[0,0]^T$ in the H-LIP dynamics; the same step in the L-S2S dynamics results in $\mathbf{x}_{k+1} = \bar{C} \neq [0,0]^T$. The physical meaning is that when the COM projects on the stance foot with the applied step size being 0, the COM state in the next step will not be zero due to the dynamics effect of swinging the leg. 


\subsection{Orbit Characterization of L-S2S }
Before synthesizing the stepping controller based on the L-S2S dynamics, we need to first characterize the desired walking behaviors. The periodic walking will be directly represented by the state of the L-S2S dynamics, which represents the pre-impact state of the robotic walking. In the following, we briefly present the characterization of the Period-1 (P1) and Period-2 (P2) orbits, which can then be composed for 3D bipedal walking \cite{xiong20213d}. 

\vspace{0.5mm}
\noindent{\textbf{Period-1 Orbit:}} P1 orbits are the one-step orbits, i.e., one step of walking completes one orbit in the continuous state-space. The desired step size of a P1 orbit is determined by the desired walking velocity $v^d$. Given a fixed step duration $T$, $u^* = v^d T$, and the corresponding state of the L-S2S is:
\begin{equation}
\label{eq:p1_L_S2S}
\mathbf{x}^* = (I - \bar{A})^{-1}( \bar{B} u^* + \bar{C}),
\end{equation}
which is solved by letting $\mathbf{x}_{k+1} =\mathbf{x}_{k} $ in Eq. \eqref{eq:rS2S}. $I$ is the identity matrix. Therefore, given a desired walking velocity, the desired pre-impact state and the desired step size of the P1 orbit are directly identified. 

\vspace{0.5mm}
\noindent{\textbf{Period-2 Orbit:}} P2 orbits are the two-step orbits, i.e., it takes two steps to complete a periodic walking. Let the subscript $_\text{L/R}$ denote the left or right legs. The P2 orbit that realizes a desired velocity $v^d$ is not unique \cite{xiong20213d}. The sum of the step sizes $u_{ \sum} := u^*_\text{L} + u^*_\text{R} = 2 v^d T$. Selecting one step size then determines the orbit. Solving $\mathbf{x}_{k+2} = \mathbf{x}_k$ with the L-S2S dynamics yields the corresponding pre-impact states: 
\begin{equation}
\label{eq:p2_L_S2S}
\mathbf{x}^*_\text{L/R} = (I - \bar{A}^2)^{-1}( (\bar{A}\bar{B} - \bar{B})u^*_{\text{L/R}} +  \bar{B} u_{ \sum} + (\bar{A}+I)\bar{C}). 
\end{equation}
Note that P2 orbits are mainly used in 3D walking \cite{xiong20213d}.  

\textbf{Remark:} The state-feedback stepping controller in Eq. \eqref{eq:HLIPstepping} can be directly applied for stepping stabilization using the L-S2S dynamics. One only needs to replace $u^\text{H}$ and $\mathbf{x}^\text{H}$ (of the H-LIP) by $\mathbf{x}^*$ and $u^*$ (of the L-S2S); $\mathbf{x}^*$ is calculated in Eq. \eqref{eq:p1_L_S2S} for P1 orbits and in Eq. \eqref{eq:p2_L_S2S} for P2 orbits. 


\section{Robust Stepping Stabilization based on L-S2S}
\label{sec:push} 
We now present the problem formulation of the robust stepping stabilization based on the L-S2S dynamics. In particular, we consider characterizing the continuous external pushes as the disturbance to the L-S2S dynamics so that robust push-rejecting stepping controllers can be synthesized. 

\subsection{Push Disturbances to S2S Dynamics}
A push is modeled as an external horizontal force $F_\text{ext}$ with a pushing duration. The Euler-Lagrangian equation of the robot in the continuous dynamics is, 
\begin{equation}
    M(q) \ddot{q} + h(q, \dot{q}) =  \tau + J_c^T F_\text{ext}(t),
\end{equation}
where $q$ is the minimum representation of the robot configuration, $M(q)$ is the inertia matrix, $h(q, \dot{q})$ is the Coriolis, centrifugal, and gravitational term, $\tau$ represents the joint motor torques, and $J_c$ is the Jacobian of the position of the push. It is obvious that $F_\text{ext}$ directly affects the continuous dynamics and then the S2S dynamics. The disturbed S2S dynamics can be represented by: 
\begin{equation}
    \mathbf{x}^-_{k+1} = \mathcal{P}(\mathbf{x}^-_k,  \tau(t)) +  \mathcal{P}^{\text{ext}}_\mathbf{x}(F_\text{ext}(t), x^-_k, \tau(t), t_0, t_F), \nonumber
\end{equation}
where $t_0, t_F$ denote the time of the start and end of the push, and $\mathcal{P}_{\text{ext}}$ stands for the influence of the push to the S2S dynamics. One can quickly realize that the push component $\mathcal{P}_{\text{ext}}$ cannot be obtained analytically. 
 \begin{figure}[b]
      \centering
      \includegraphics[width = 0.65\columnwidth]{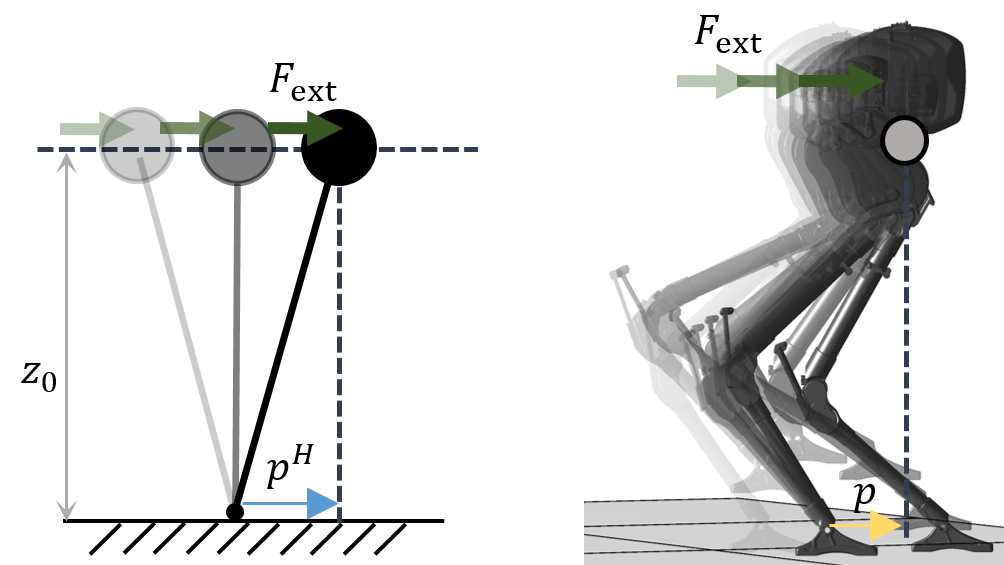}
      \caption{Illustration of the external push force that is applied to the pendulum model (left) and the robot (right).}
      \label{fig:LIPapprox}
\end{figure} 
To characterized $\mathcal{P}^{\text{ext}}$ approximately, we use the H-LIP to approximate the robot dynamics again (illustrated in Fig. \ref{fig:LIPapprox}). Without losing generality, we also make a few assumptions for simplification. We assume that the horizontal push force is constant and only happens in the entire single support phase (SSP). Consecutive pushes do not happen within $N_\text{push} >2$ steps. Consider the dynamics of the H-LIP in the SSP with a push: 
\begin{equation}
  \frac{d}{dt} \begin{bmatrix} p^H \\ v^H \end{bmatrix} = \begin{bmatrix} 0 & 1 \\ \lambda^2 & 0\end{bmatrix}\begin{bmatrix} p^H \\ v^H \end{bmatrix} + \begin{bmatrix} 0 \\ \frac{1}{m} \end{bmatrix} F_{\text{ext}},
\end{equation}
where $m$ is the mass, and $\lambda  = \sqrt{g\slash z_0}$ with $z_0$ being the constant height of the COM \cite{xiong20213d}. Here, $p^H$ is the horizontal position of the mass, and $v^H = \dot{p}^H$ is the horizontal velocity in continuous time. This linear time invariant system has a closed-form solution. Moreover, the disturbed S2S of the H-LIP becomes $ \mathbf{x}^H_{k+1} = A  \mathbf{x}^H_{k+1} + B u^H + w_\text{ext}$ with
\begin{align}
     w_\text{ext} = F_\text{ext} \textstyle \frac{\text{sinh}(T_S \lambda)}{m\lambda} \begin{bmatrix} \frac{1}{\sigma_1} & 1\end{bmatrix}^T, \label{eq:Fext2wext}
\end{align}
where $\sigma_1 = \lambda \text{coth}(\frac{T_S }{2}\lambda)$ is the orbital slope \cite{xiong20213d} of P1 orbits of the H-LIP. We now have a mapping in Eq. \eqref{eq:Fext2wext} from the external push $F_\text{ext}$ to the disturbance $w_\text{ext}$ to the S2S dynamics. Assuming bounded $F_\text{ext}$, $w_\text{ext}$ is bounded to a set, i.e., $w_\text{ext} \in \mathbf{W}_\text{ext}$. We will use this mapping to approximately quantify the push disturbance to the robot S2S dynamics. One can also apply data-driven approaches to characterize $\mathcal{P}_{\text{ext}}$ from disturbed walking data in simulation or experiment. Regardless, the following robust stepping synthesis problem remains the same as long as we consider robustly stabilizing bounded push disturbances. 


\subsection{Problem Formulation for Robust Stepping Stabilization}
With the L-S2S dynamics and its orbit characterizations, a similar state-based feedback controller in Eq. \eqref{eq:HLIPstepping} can be directly applied for disturbance-free walking realization. Here, we want to synthesize a stepping controller that not only realizes walking that is robust to external push disturbances but respects the kinematic feasibility of the robot. 


\subsubsection{Objective} Suppose we want to achieve a certain walking behavior described by the L-S2S dynamics. Let $u^*$ be the desired step size and $\mathbf{x}^*$ be the desired pre-impact horizontal COM state. The general formulation for stabilizing the robot to the desired walking of the L-S2S dynamics is:
 \begin{equation}
 \label{eq:RS2SsteppingGeneral}
u = u^* + u^e,
\end{equation}
where $u$ is the step size that will be applied to the robot, and $u^e$ is the step size that we need to synthesize. The closed-loop system of the S2S dynamics of the robot becomes:
\begin{equation}
\label{eq:cl_sys}
    \mathbf{x}_{k+1} = \bar{A} \mathbf{x}_{{k}}  +  \bar{B} (u^*_k + u^e_k) + \bar{C} +  \epsilon  +  w_\text{ext}. 
\end{equation}
Similar to the previous construction in Section \ref{sec:prelim}, we define the error state as $\mathbf{e} = \mathbf{x} - \mathbf{x}^*$. Subtracting Eq. \eqref{eq:cl_sys} by Eq. \eqref{eq:rS2S} yields the error Dynamics:
\begin{align}
\label{eq:errorDynamics}
 \mathbf{e}_{k+1} = \bar{A} \mathbf{e}_{k} + \bar{B} u^e +   w, 
\end{align}
where $w: = w_\text{ext} + \epsilon$. $w \in \mathbf{W}: = \mathbf{W}_\text{ext} \oplus \mathbf{D}$. Thus, the goal is to synthesize a controller that stabilizes $\mathbf{e} \rightarrow 0$ subject to the new disturbance $w$ (model discrepancy $\epsilon$ $+$ external push disturbance $w_\text{ext}$). 



\subsubsection{Kinematic Constraints} We also need the stepping controller to be aware of the kinematic feasibility of the robot. It is obvious that the robot cannot realize arbitrary step size during walking. The horizontal COM position w.r.t. to the stance foot also belongs to a bounded set. We leave the identification of the constraints in the Appendix. Generally speaking, each set can be identified via sampling the kinematic space of the robot subject to certain vertical COM height and swing foot positions. The state and input constraints are concisely defined as follows.

\sectionnewit{State Constraint:} Let $\mathbf{X}$ represent the set of the horizontal COM state of the robot. Then, $\mathbf{x} \in\mathbf{X}$,
$\mathbf{e} \in \mathbf{X}^e = \mathbf{X} -  \mathbf{x}^*.$

\sectionnewit{Input Constraint:} Let $\mathbf{U}$ represent the set of feasible step sizes; $u \in \mathbf{U}$. From Eq. \eqref{eq:RS2SsteppingGeneral}, $\mathbf{U}^e = \mathbf{U} - u^*$. Therefore, the stepping controller should be synthesized subject to $u^e \in \mathbf{U}^e$.

Therefore, given a desired steady state walking from the L-S2S, $u^*$ and $\mathbf{x}^*$ have a one-to-one mapping in closed-form; the state and input constraints are then identified from the robot kinematics for the robust stepping synthesis. 



\section{SLS For Robust Stepping Stabilization}
\label{sec:SLS}
To realize the robust stepping stabilization, we use system-level-synthesis (SLS) to design a dynamic feedback controller that renders the closed-loop system 1) stable with finite impulse response (FIR), i.e., recovers from any external disturbance in finite number of steps, and 2) satisfies the state and input constraints under any disturbance in $\mathbf{W}$. 
We begin by representing a brief overview of the SLS framework and then apply it to the stepping problem. The readers are referred to~\cite{anderson2019system, WanMD18} for a more complete picture of SLS.

\subsection{Review: System-Level-Synthesis}
Consider a general discrete linear dynamic system
\begin{equation}\label{eq:dyn}
    x_{k+1}=A x_k + B u_k+w,
\end{equation}
where $x\in\mathbb{R}^n$ is the state, $u\in\mathbb{R}^m$ is the control input and $w\in\mathbb{R}^n$ is the disturbance. The SLS takes into account the system dynamics and directly optimizes over the closed-loop map from disturbance to state and control action:
\begin{align} \label{eq:sysresp}
\begin{bmatrix} \tf x \\ \tf u \end{bmatrix} = \begin{bmatrix} \Phix \\ \Phiu \end{bmatrix}\tf w,
\end{align}
where $\tf x$, $\tf u$, and $\tf w$ are the state, input, and disturbance signals, and
\begin{align}
\Phix = (zI-A -B\tf K)^{-1},~
\Phiu = \tf K(zI-A -B\tf K)^{-1}. \nonumber
\end{align}

The control synthesis problem is to design a dynamic state-feedback policy $\tf u = \tf K \tf x$, and SLS does so by directly optimizing the closed-loop system responses $\SFpair$ (by choice of $\tf K$). Any stable and strictly-proper transfer matrices $\SFpair$ that satisfy the affine expression
\begin{equation}\label{eq:affine}
\begin{bmatrix} zI - A & -B\end{bmatrix} \begin{bmatrix} \Phix \\ \Phiu \end{bmatrix}=I,
\end{equation}
can be used to construct an internally stabilizing controller $\tf K = \Phiu \Phix^{-1}$. To make the control synthesis more intuitive, we work in the time domain with a convolutional representation of the system response given by
\begin{align}
x_k = \textstyle  \sum_{i=1}^{\infty}\Phixt[i]w(k-i),  ~
u_k =   \sum_{i=1}^{\infty}\Phiut[i]w(k-i). \nonumber
\end{align}
The relationship between $\Phix,\Phiu$ and $\Phixt [1],\hdots,$ and $\Phiut [1],\hdots$ is given through the spectral decomposition of a transfer matrix: $\Phix = \sum_{i=0}^{\infty} \Phixt[i]z^{-i}$, $\Phiu = \sum_{i=0}^{\infty} \Phiut[i]z^{-i}$. To make the synthesis tractable, typically a finite impulse response (FIR) constraint is added, i.e., $\Phix = \sum_{i=0}^{N_\text{F}} \Phixt[i]z^{-i}$, $\Phiu = \sum_{i=0}^{N_\text{F}} \Phiut[i]z^{-i}$, where $N_\text{F}$ is the FIR horizon. 

The SLS framework was proposed to handle locality and sparsity constraints for a distributed control problems. Here we do not have locality constraints, instead, we are concerned with the input and state constraints in the form of linear inequalities. We adopt the solution proposed in \cite{chen2019system} where the goal is to synthesize a controller such that for all $\mathbf{w}$ that satisfies $G[k] w_k\le g[k],k=0,...,N_\text{F}-1$, the closed loop system satisfies $H[k]\begin{bmatrix}
x_k\\u_k
\end{bmatrix}\le h[k]$ for $k=1,...,N_\text{F}$. The SLS synthesis problem with state and input constraints is formulated as a robust optimization and solved with the following equivalent linear program:
\begin{equation}\label{eq:conv_problem}
  \begin{aligned}
  \mathop {\min }\limits_{{\Phix},{\Phiu},\Lambda\ge0} \quad &J({\Phix},{\Phiu})&\\
 \rm{s.t.}\; \quad\quad ~& \textstyle \begin{bmatrix}
    Iz-A & -B
  \end{bmatrix}
  \begin{bmatrix}
    \Phix \\
    \Phiu
  \end{bmatrix}=I \\
  &\Phix,\Phiu\in  \textstyle \frac{1}{z}\mathcal{RH}_\infty\cap FIR(N_\text{F}), \\
  &\forall i=1,...,N_\text{F},\; j=0,..,i-1\\
  & H[i]\Phi[i-j]  = \Lambda[i,j]G[j],\\
       & \textstyle \sum_{j=0}^{i-1}\Lambda[i,j] g[j]\le h[i],
  \end{aligned}
\end{equation}
where $\frac{1}{z}\mathcal{RH}_\infty$ denotes the space of stable transfer functions, and $FIR(N_\text{F})$ denotes the transfer functions with a FIR horizon $N_\text{F}$. Further details can be found in \cite{chen2019system}.

\subsection{Application to Bipedal Robotic Push Rejection}
We apply SLS on the error dynamics where the goal is to keep $\mathbf{e}_k\in \mathbf{X}^e$ with $u^e_k\in\mathbf{U}^e$ for all $k\ge 0$.

To handle external push via SLS, we need $N_\text{push}\ge N_\text{F}$, i.e., the SLS controller should recover from the external push before the next push happens. We consider the following profile of bounds on the disturbance signal $\mathbf{w}$:
\begin{align}\label{eq:w_bound}
    &w_i\in\mathcal{S}_0,&i&=0\\
    &w_i\in\mathbf{W}_\text{ext}\oplus\mathbf{D}, &i&=1\\
    &w_i\in \mathbf{D},&i&=2,...,N_\text{F}-1
\end{align}
where $\mathcal{S}_0$ is the set of possible initial state $\mathbf{e}_1$ since in SLS synthesis, $w_0$ is simply the initial state $\mathbf{e}_1$. The push happens at $i=1$, therefore the possible disturbance at $i=1$ is the modeling error plus the external push. For $2\le i\le N_\text{F}-1$, the disturbance is simply the modeling error since no external push is allowed before $N_\text{push}-1$ steps after the last push.

The input constraint is fairly simple, for all $1\le i\le N_\text{F}$, $u^e \in\mathbf{U}^e$. The state constraint is set as follows:
\begin{align}
    &\mathbf{e}_i\in\mathbf{X}^e, &i&=1,...,N_\text{F}-1\\
    &\mathbf{e}_i\in\mathcal{S}_0,&i&=N_\text{F}
\end{align}
The last constraint makes sure that the error after $N_\text{F}-1$ steps of the push is within $\mathcal{S}_0$. Since all the above constraints are linear inequalities, the SLS can be solved with Eq. \eqref{eq:conv_problem}.
\begin{theorem}
Suppose the SLS synthesis is feasible with $N_\text{F}\le N_\text{push}$ and some $\mathcal{S}_0\subseteq\mathbf{X}^e$, the initial condition of the robot satisfies $\mathbf{e}_0\in\mathcal{S}_0$, then the closed-loop system satisfies $\forall k\ge0, \mathbf{e}_k\in\mathbf{X}^e, u^e_k\in\mathbf{U}^e$.
\end{theorem}

\begin{proof}
The proof is by induction. Let $t_i$ be the timing of the $i$-th push. Suppose $\mathbf{e}_{t_i}\in\mathcal{S}_0$, then the disturbance signal from $t_i$ to $t_i+N_\text{F}$ satisfies Eq. \eqref{eq:w_bound}. Therefore, for all $k=t_i,...,t_i+N_\text{F}$, $\mathbf{e}_k\in\mathbf{X}^e$, $u_k\in\mathbf{U}^e$, and $\mathbf{e}_{t_i+N_\text{F}}\in\mathcal{S}_0$. For all $k=t_i+N_\text{F}+1,...,t_{i+1}$ since the disturbance signal within $[k-N_\text{F}+1,k]$ satisfies Eq. \eqref{eq:w_bound} (no push happens in the meantime), $\mathbf{e}_k\in\mathcal{S}_0\subseteq\mathbf{X}^e$, $u_k\in\mathbf{U}^e$. Therefore, $\mathbf{e}_{t_i}\in\mathcal{S}_0$. Since the initial condition satisfies $\mathbf{e}_0\in\mathcal{S}_0$, the proof is completed by induction.
\end{proof}

\section{Evaluation on Bipedal Walking Robots}
\label{sec:results}
In this section, we evaluate the proposed SLS-based stepping controller on the walking control of the planar bipedal robot AMBER and the 3D bipedal robot Cassie \cite{xiong20213d}. The purpose of the evaluation on different robots is to demonstrate the generality of the proposed approach. Despite that the two robots have significantly different morphology and inertia distributions, the proposed stepping controller effectively stabilizes the robotic walking to desired behaviors with strong robustness against push disturbances.

\subsection{Robot Models and Control Strategy}
{\small \sectionnew{AMBER}} is a planarized bipedal robot with two point-feet, two hip, and two knee joints. The walking is a single-domain hybrid system consisting of a SSP and a discrete impact event. The point-foot contact with the ground renders the continuous dynamics of the angular momentum about the foot (approximately the horizontal COM) underactuated. Following the H-LIP approach in \cite{xiong20213d}, we define the output of walking to be the combination of the torso angle $q_\text{torso}$, the vertical COM position and the swing foot position $[x_\text{sw}, z_\text{sw}]$:
\begin{align}
    \mathcal{Y} = & \begin{bmatrix}
     z_\text{COM} &
     x_\text{sw} &
     z_\text{sw} &
    q_\text{torso}
    \end{bmatrix}^T \nonumber \\ &-
 \label{eq:AMBERoutput}
    \begin{bmatrix}
     z^d_\text{COM}  &
  x^d_\text{sw}(q, \dot{q}, t) &
  z^d_\text{sw}(t) &
      q^{d}_\text{torso}
    \end{bmatrix}^T
\end{align}
where $z^d_\text{COM} = z_0$ is the desired constant COM height, $q^{d}_\text{pelvis}$ is the desired pelvis angle, and $[x^d_\text{sw}, z^d_\text{sw}]$ represent the desired swing foot position w.r.t. the stance foot. $z^d_\text{sw}(t)$ is designed to periodically lift off and strike the ground. $x^d_\text{sw}$ is constructed to achieve the step size $u$. We apply the output construction in \cite{xiong20213d}; e.g., the desired horizontal swing foot is:
$ x^d_\text{sw} = (1-c(t))x^+_\text{sw} +c(t) u$, 
where $x^+_\text{sw}$ is the horizontal position of the swing foot in the beginning of the step, and $c(t)$ is a Bézier polynomial that transits from 0 to 1 within the step. 


\sectionnew{Cassie} is a 3D bipedal robot, the walking of which is modeled as a two-domain hybrid dynamical system that contains a SSP and a double support phase (DSP). 
Due to the space constraint, we refer the readers to \cite{xiong20213d} for more details on the robot model and the output definitions. Loosely speaking, the output is designed to be the combination of the torso orientation, the vertical COM position, the swing foot position, and the swing foot orientation. Similarly, the desired swing foot trajectories are designed so that the swing foot periodically steps to achieve the desired step size $u$.

 \begin{figure}[t]
      \centering
      \includegraphics[width = .7\columnwidth]{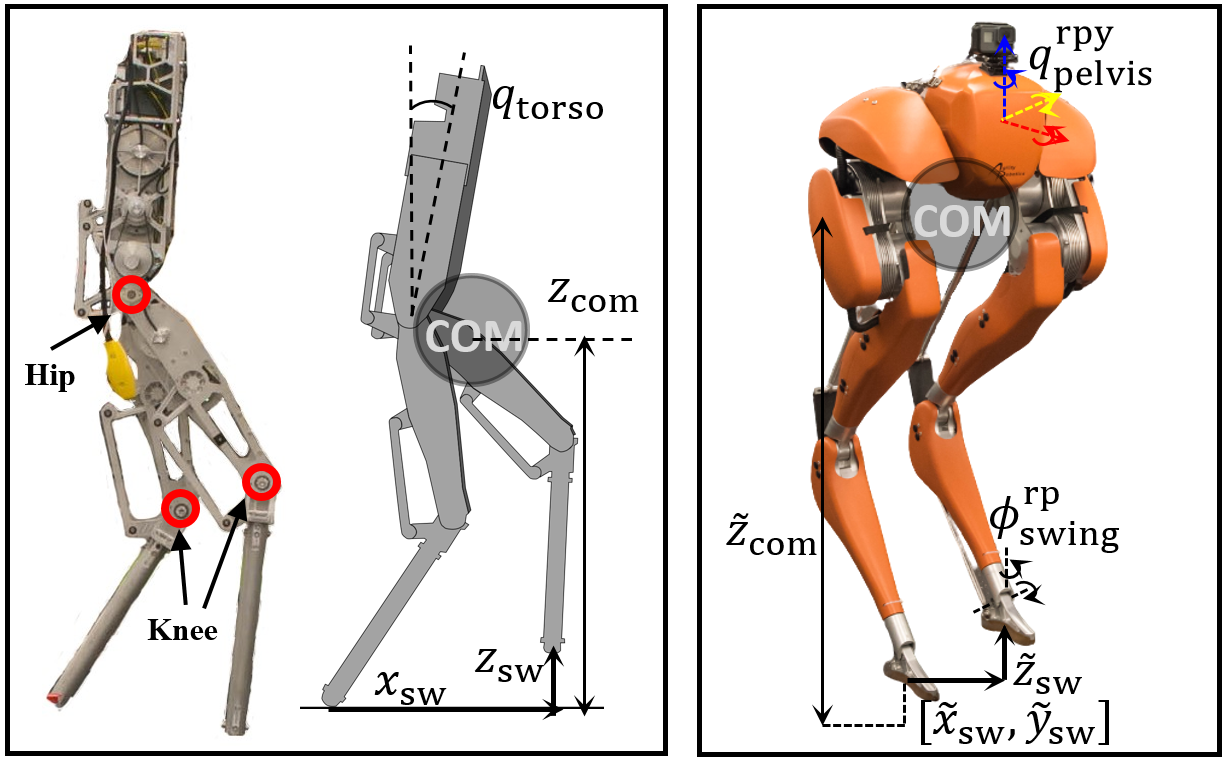}
      \caption{The robot AMBER and Cassie with their output definitions \cite{xiong20213d}.}
      \label{fig:obstacle}
\end{figure}

\sectionnew{Control Realization}: We first apply the H-LIP stepping in \cite{xiong20213d} to generate enough walking data for learning. The step duration $T$, desired vertical COM height $z_0$, and parameters in the trajectory synthesis and output stabilization remain the same for all walking behaviors. Then, the discrete horizontal COM states $\mathbf{x}$ and the actual step size $u$ in the S2S dynamics are extracted. We then solve the linear program in Eq. \eqref{eq:L-S2S} to get the L-S2S dynamics. The SLS-based stepping controllers are synthesized on the L-S2S dynamics with bounded push disturbance. The state and input constraints are characterized from the kinematic limits of the robot. Finally, we evaluate the approach on the robot walking with external pushes. The procedures are summarized in Algorithm 1. The video of the results can be seen online\footnote{\href{https://www.xiaobinxiong.info/research/sls}{https://www.xiaobinxiong.info/research/sls}}.

\begin{algorithm}\caption{System-Level-Synthesis for Robust Walking}\label{alg::SimLoop}
 \begin{algorithmic}[1]
 \renewcommand{\algorithmicrequire}{\textbf{Initialization:}}
 \renewcommand{\algorithmicensure}{\textbf{Customization:}}
 \REQUIRE \textit{Gait Parameters}: $z_0,~T$. $N_\text{push}$, $N_\text{F}$ ($N_\text{F} \leq N_\text{push}$),  $F_\text{ext}$, and \textit{Output Stabilizing Parameters}.
\STATE Generate walking data $\leftarrow$ \textit{H-LIP based approach} \cite{xiong20213d}
\STATE Learn the L-S2S in Eq. \eqref{eq:L-S2S}.
\STATE Optimize SLS problem in Eq. \eqref{eq:conv_problem}
\STATE Desired walking $\mathbf{x}^*, u^*$ from the L-S2S
\WHILE {Simulation/Control Loop}
\IF {SSP}
\STATE Apply Push Forces $F_\text{ext}$ at a certain step 
\STATE Solve $u^e$ from SLS and $u$ from Eq. \eqref{eq:RS2SsteppingGeneral} 
\STATE Construct output $\mathcal{Y}$ e.g. in Eq. \eqref{eq:AMBERoutput} 
\ENDIF \\
\STATE Output Stabilization
\ENDWHILE
 \end{algorithmic}
 \end{algorithm}

  \begin{figure}[t]
      \centering
      \includegraphics[width = .85\columnwidth]{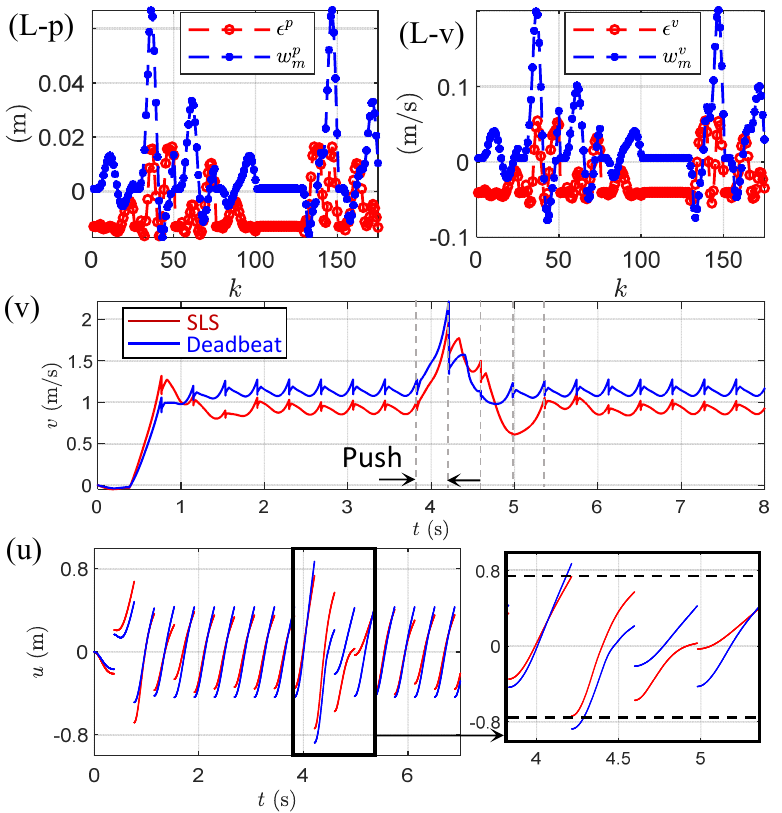}
      \caption{Push rejection on AMBER: (L-p, L-v) the learning residual $\epsilon$ of the L-S2S dynamics (red) compared to the model discrepancy $w_m$ of the H-LIP approximation, (v) the velocity trajectories of the horizontal COM in the sagittal plane under the SLS controller (red) and the deadbeat controller (blue), and (u) the input profiles of the two controllers. }
      \label{fig:AMBER}
\end{figure}

\subsection{Evaluation on AMBER}
We first evaluate the approach on AMBER. The gait parameters are chosen to be $z_0 = 0.7$m, $T = 0.4$s and $q^d_\text{torso} = 0$. As for the L-S2S dynamics, Fig. \ref{fig:AMBER} shows the learning residuals $\epsilon$. In comparison, we also show the model discrepancy of $w_m$ before the learning. The learning significantly reduced the discrepancy with $d^* = [0.0164 \text{m}, 0.0546 \text{m/s}]$, which is very small in practice. 

We then apply the SLS controller to realize walking with $v^d = 1$m/s from a static configuration. For simplicity, we select $\mathbf{U} =[-0.7, 0.7]$m. We choose $F_\text{ext} = 50$N and apply it on the robot for a whole step. The FIR horizon is chosen to be $N_\text{F} = 4$, thus the push disturbance is supposed to be stabilized after 4 steps of walking. Fig. \ref{fig:AMBER} shows the velocity profile under the stepping controller. The robot rejects the push disturbance using 4 steps. The step size and the horizontal COM state are within their given sets

In comparison, we also applied the deadbeat controller for push rejection. The robot is stabilized with fewer steps, but the kinematic constraints on the step sizes are violated as shown in Fig. \ref{fig:AMBER} (u). Thus, the SLS controller is a better choice. Note that the desired step size $u$ is continuously constructed using the current horizontal COM state since the pre-impact one is not known before the impact happens \cite{xiong20213d}. 

\subsection{Evaluation on Cassie}
The procedure of the evaluation on Cassie is similar to that on AMBER. The main difference is that Cassie is a 3D robot, thus the stepping stabilization is decoupled into its sagittal and coronal plane. The decoupling has been realized in \cite{xiong20213d} using the H-LIP based approach. Here, we select the walking in the sagittal plane for the evaluation of the SLS controller. $z_0 = 0.9$m, $T = 0.35$s, $\mathbf{U} = [-0.7, 0.7]$m, and $N_\text{F} = 4$. The model discrepancy $\epsilon$ is shown in Fig. \ref{fig:Cassie}, and the bound is $d^* = [0.0119 \text{m}, 0.0561 \text{m/s}]$, which is also very small. We consider realizing a push recovery behavior with $v^d = 0$. We increase the push force to be $F_\text{ext} = 120$N and apply it to the robot for one step. Fig. \ref{fig:Cassie} shows the simulated push recovery behaviors using the SLS controller and also the deadbeat controller. Unsurprisingly, although both controllers stabilized the push disturbance, the SLS controller outperforms the original deadbeat controller in terms of respecting the kinematic constraints. The maximum step size taken by the SLS is $0.68$m and that taken by the deadbeat controller is $0.83$m, which exceeds $\mathbf{U}$. We also realized this approach on the hardware of the robot (see Fig. \ref{fig:overview}). The procedure is the same. Although we cannot push the robot with a precise force and duration, the SLS approach is shown to be effective to keep the robot balanced while rejecting mild kicks under finite steps ($N_\text{F} = 4$).

 
  \begin{figure}[t]
      \centering
      \includegraphics[width = .85\columnwidth]{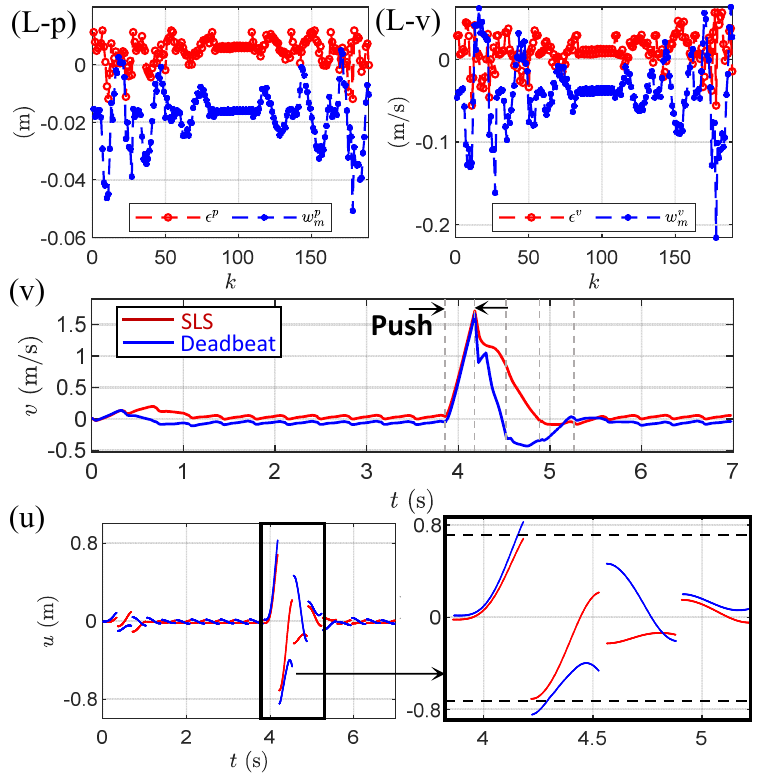}
      \caption{Push recovery on Cassie. The presentations of individual sub-figures are identical to these in Fig. \ref{fig:AMBER}. }
      \label{fig:Cassie}
\end{figure}

\subsection{Discussion}
\vspace{-1mm}
\sectionnewit{Implications:}
The evaluations on two different robots further confirm that a linear approximation to the discrete S2S dynamics of walking is sufficient for walking synthesis, despite the differences in the hybrid structure, inertia distributions of the robot, and impact conditions. The learning on the S2S dynamics is efficiently solved as a convex program with a small set of variables. The simplicity of the L-S2S dynamics also facilitates feedback controllers to be easily synthesized and implemented on complex robots.  



\sectionnewit{Limitations:}
The set of walking data that is used for learning is assumed to be generated with the same parameters in the H-LIP based approach. The parameters include those in the output design, which are the coefficients in the Bézier polynomial, the step frequency, and the vertical COM height. The changes of those parameters are supposed to change the L-S2S dynamics. As a result, the L-S2S that learned from this set of parameters uses the same set of parameters and thus limits the realizable walking behaviors on the robot. For instance, the step frequency cannot be changed under pushes. 


\sectionnewit{Future Work:}
Here, we find a very practical application of the SLS to individual bipedal walking robots. It is possible that this framework can be applied to multi-robot collaborations where the robots are legged (bipedal) robots; the interaction dynamics with the legged systems can possibly be approximated via linear systems. We will continue exploring in this direction of discrete stepping control on legged robots in the future.  

\section{Conclusion}
\label{sec:conclude}
To conclude, we present a system-level-synthesis (SLS) based controller on the learned step-to-step (L-S2S) dynamics for stepping stabilization on bipedal walking robots. The application of the SLS on the L-S2S includes the kinematic constraints of the step sizes in the synthesis and rejects push disturbance in finite steps. The approach is evaluated on two different bipedal robots for realizing walking with desired velocities and robust disturbance rejections, showing great premises of effective, efficient, and robust control synthesis on high-dimensional complex bipedal robotic systems.   

\section{Appendix: Kinematic Constraints}

  \begin{figure}[b]
      \centering
      \includegraphics[width = 0.7\columnwidth]{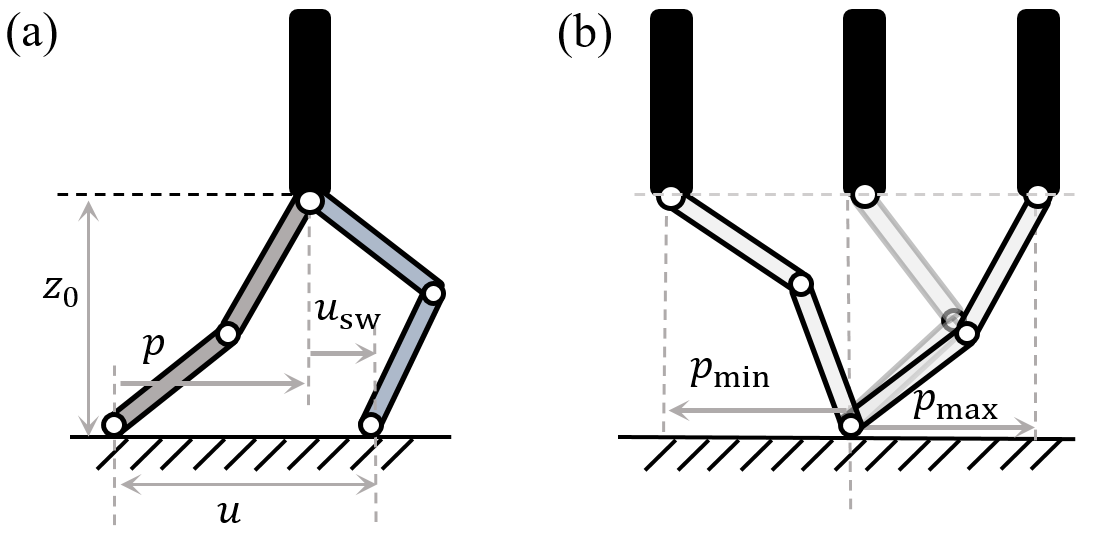}
      \caption{Illustration of the kinematic constraints on a simple robot: (a) definitions and (b) range of motion under the constant COM height.}
      \label{fig:kinematic}
 \end{figure} 
The kinematic constraints of the robot can be converted to the state and input constraints of the S2S dynamics. For simplicity, we use a simple robot (see Fig. \ref{fig:kinematic}) assuming that the COM is located at the hip joint, which actually is a practical assumption on real robots (that are top heavy, e.g., Atlas \cite{feng2015optimization} and Cassie \cite{xiong20213d}). Let $u^\text{sw}$ be the position of the swing foot w.r.t. the hip, and thus $u = p + u^\text{sw}$. Given $z_0$, let $p_\text{min}, p_\text{max}$ be the bounded value for $p$ under the kinematic constraints defined by the robot joints. Then, $p \in \mathbf{P} := [p_\text{min}, p_\text{max}]$, and $u^\text{sw} \in \mathbf{U}^\text{sw}: = [- p_\text{max}, -p_\text{min}]$. To consider the kinematic feasibility, it is more convenient to use $u^\text{sw}$ as the input. The robot S2S dynamics then is:
$\mathbf{x}_{k+1} 
= \tilde{\bar{A}} \mathbf{x}_{{k}} +  \bar{B}u^\text{sw}_k + \bar{C} +  w. 
$
where $\tilde{\bar{A}} = \bar{A}  + \bar{B}[1~~0]$. The L-S2S dynamics with $u_\text{sw}$ as the input then becomes: 
$\mathbf{x}_{k+1} = \tilde{\bar{A}} \mathbf{x}_{{k}} +  \bar{B}u^\text{sw}_k + \bar{C} +  w. 
$
The dynamics are of similar forms of these with $u$ being the input; the only difference is that $\bar{A}$ becomes $\tilde{\bar{A}}$. All the previous derivations apply by replacing $\bar{A}$ by $\tilde{\bar{A}}$. Thus, for convenience, we do not differentiate which form of input $u$ or $u^\text{sw}$ we use. Instead, in either way the kinematic constraints are represented by the input and state constraints in the control synthesis.  

\IEEEpeerreviewmaketitle

\bibliographystyle{IEEEtran}
\bibliography{references}

\end{document}